\documentclass[imslayout]{imsart}

\RequirePackage[OT1]{fontenc}
\RequirePackage{amsmath, amssymb, amsthm, algorithm, algorithmic}
\RequirePackage[colorlinks,citecolor=blue,urlcolor=blue]{hyperref}
\usepackage{wrapfig}
\usepackage{graphicx}
\RequirePackage[numbers]{natbib}
\usepackage{fullpage}
\startlocaldefs
\numberwithin{equation}{section}
\theoremstyle{plain}

\theoremstyle{remark}

\endlocaldefs

\newtheorem{theorem}{Theorem} 
\newtheorem{lemma}[theorem]{Lemma} 
\newtheorem{proposition}[theorem]{Proposition} 
\newtheorem{corollary}[theorem]{Corollary}

\newtheorem{condition}{Sufficient Condition} 
 
\newtheorem{example}[theorem]{Example}

\newcommand{\Prob}[1]{\ensuremath{{\mathbb P} \left[#1\right]}} 
\newcommand{\E}{\ensuremath{\mathbb E}} 
\newcommand{\R}{\ensuremath{\mathbb R}}

\newcommand{\K}{\ensuremath{\mathcal K}}
\newcommand{\cN}{\ensuremath{\mathcal N}}
\newcommand{\Y}{\ensuremath{\mathcal Y}}
 
\newcommand{\tr}{\ensuremath{{\scriptscriptstyle\mathsf{T}}}}
\newcommand{\inner}[1]{\left\langle #1 \right\rangle}

\newcommand{\reals}{\ensuremath{\mathbb R}}
\newcommand{\norm}[1]{\ensuremath{\left\|#1\right\|}}
\newcommand{\sig}[2]{\ensuremath{\sigma_{{\scriptscriptstyle #1,#2}}}}
\def\deq{\triangleq}
\newcommand{\Ps}{\ensuremath{\mathsf P}} 
\newcommand{\loss}{\ensuremath{\boldsymbol \ell}} 
\newcommand{\Reg}{\ensuremath{\text{Reg}}} 

\newcommand{\ra}{\ensuremath{\rightarrow}}

\newcommand{\p}{\mathbb{P}}

\newcommand{\LL}{\mathcal L} 
\newcommand{\MM}{\mathcal M}

\def\eps{{\epsilon}}

\def\g{G}

\begin{document}

	\begin{frontmatter}
	\title{Efficient Sampling from Time-Varying Log-Concave Distributions}
	\runtitle{Sampling from Time-Varying Distributions}

	\begin{aug}
	\author{\fnms{Hariharan} \snm{Narayanan}\ead[label=e2]{harin@uw.edu}}
	\and
	\author{\fnms{Alexander} \snm{Rakhlin}\thanksref{t1}\ead[label=e1]{rakhlin@wharton.upenn.edu}}

	\thankstext{t1}{Supported in part by NSF under grant CAREER DMS-0954737. }
	% \thankstext{t2}{First supporter of the project}
	% \thankstext{t3}{Second supporter of the project}
	% \runauthor{F. Author et al.}

	%\affiliation{University of Pennsylvania\thanksmark{m1}, TTIC-Chicago\thanksmark{m2}, and University of Texas\thanksmark{m3}}
	\affiliation{University of Washington and University of Pennsylvania}

	\address{Department of Statistics and Department of Mathematics\\
	University of Washington\\
	\printead{e2}\\
	\phantom{E-mail:\ harin@uw.edu}
	}

	\address{Department of Statistics, The Wharton School\\
	University of Pennsylvania\\
	\printead{e1}\\
	\phantom{E-mail:\ rakhlin@wharton.upenn.edu}
	}
	\end{aug}

	\begin{abstract}
		We propose a computationally efficient random walk on a convex body which rapidly mixes and closely tracks a time-varying log-concave distribution. We develop general theoretical guarantees on the required number of steps; this number can be calculated on the fly according to the distance from and the shape of the next distribution. We then illustrate the technique on several examples. Within the context of exponential families, the proposed method produces samples from a posterior distribution which is updated as data arrive in a streaming fashion. The sampling technique can be used to track time-varying truncated distributions, as well as to obtain samples from a changing mixture model, fitted in a streaming fashion to data. In the setting of linear optimization, the proposed method has oracle complexity with best known dependence on the dimension for certain geometries. In the context of online learning and repeated games, the algorithm is an efficient method for implementing no-regret mixture forecasting strategies. Remarkably, in some of these examples, only one step of the random walk is needed to track the next distribution.
	\end{abstract}

	\begin{keyword}[class=AMS]
	\kwd[Primary ]{60K35}
	\kwd{60K35}
	\kwd[; secondary ]{60K35}
	\end{keyword}

	\begin{keyword}
	\kwd{sample}
	\kwd{\LaTeXe}
	\end{keyword}

	\end{frontmatter}

\section{Introduction}

Let $\K$ be a compact convex subset of $\reals^d$ with non-empty interior. Let $\mu_0,\ldots,\mu_t,\ldots$ be a sequence of probability measures with support on $\K$. Suppose each probability distribution $\mu_t$ has a density
\begin{align}
	\label{eq:def_mu_t}
	\frac{d\mu_t(x)}{dx} = \frac{e^{-s_t(x)} }{Z_t}
~\mbox{,}~~~~~
Z_t = \int_{x\in\K} e^{-s_t(x)} dx
\end{align}
with respect to the Lebesgue measure, where each $s_t(x)$ is a convex function on $\K$. This paper proposes a Markov Chain Monte Carlo method for sequentially sampling from these distributions. The method comes with strong mixing time guarantees, and is shown to be applicable to a variety of problems. Observe that, by definition, the distributions $\mu_t$ are \emph{log-concave}, and thus our work falls within the emerging body of literature on sampling from log-concave distributions.

The problem of sampling from distributions arises in many areas of statistics, most notably in Bayesian inference \cite{robert2004monte}. In particular,  Sequential Monte Carlo methods \cite{doucet2001sequential} aim to sample from time-varying distributions. The need for such methods arises, for instance, in the case of online arrival of data: it is desirable to be able to update the posterior distribution at a low computational cost. If the distributions are changing ``slowly'' with time, sequential methods can re-use samples from the previous distribution and perform certain re-weighting to track the next distribution, thus saving computational resources. These ideas are exploited in particle filtering methods (see \cite{chopin2002sequential,doucet2001sequential} and references therein). Beyond Bayesian inference, other applications of sampling from distributions include simulated annealing, global optimization, and regret minimization.

The main critique of the MCMC methods is, in many situations, the lack of mixing time analysis. In practice, the number of steps of the chain required to obtain an honest sample from a distribution is mostly calculated based on heuristics. There is a growing body of literature that presents exceptions to these heuristic approaches. Coupling methods, spectral gap methods, as well as the more recent study of positive Ricci curvature, yield geometric decrease of the distance to the desired stationary distribution -- a property known as \emph{geometric ergodicity}. The most well-understood cases in this context are those with a finite or countable state space (see \cite{meyn2009markov, diaconis2009markov}). In contrast, we are interested in a random walk on a non-discrete set.

This paper is focused on a particular circle of problems defined via log-concave distributions. These distributions constitute an important subset of the set of unimodal distributions, a fact that has been recognized within Statistics (see e.g. \cite{walther2009inference}). We are not the first to study mixing times for such distributions: this line of work started with the breakthrough paper of \cite{DyeFriKan91}, followed by a series of improvements \cite{frieze1994sampling,Lovasz,LovVem06simulated,LovVem07geometry}. However, the recent advances in \cite{kannan2012random} on sampling from convex bodies give an edge to obtaining stronger guarantees. In particular, we show that we can provably track a changing distribution with a small number (or even only \emph{one step}) of a random walk, provided that the distribution changes slowly enough. Such a result seems out of reach with other random walk methods due to the lack of scale-free bounds on conductance. Interestingly, the idea of tracking a changing distribution with only one step parallels the technique of following a central path in the theory of interior point methods for optimization. 

We assume that we can compute a {\em self-concordant barrier} (see Section~\ref{sec:apps} and Appendix~\ref{sec:self_conc_def}) for the set $\K$, a requirement that is satisfied in many cases of interest. For instance, the self-concordant barrier can be readily computed in closed form if $\K$ is defined via linear and quadratic constraints. While the availability of the barrier is a stronger assumption than, for instance, access to a separation oracle for $\K$, the barrier gives a better handle on the geometry of the space and yields fast mixing of the Markov chain.

In Section~\ref{sec:apps}, we illustrate the method within several diverse application domains. As one of the examples, we consider the problem of updating the posterior with respect to a conjugate prior in an exponential family, where the parameter is taking values in a space of a fixed dimensionality given by the sufficient statistics. The constraints then constitute a prior knowledge about the possible location of the parameter. As another example, we consider sampling from a time-varying truncated distribution, as well as the extension to sampling from mixture models fitted to streaming data.
We employ the sampling technique to the classical problem of linear optimization via simulated annealing. The final example concerns the problem of regret minimization where the log-concave distribution arises naturally from the exponential weighting scheme.

The paper is organized as follows. In the next section we study the geometry of the set $\K$ induced by a self-concordant barrier and prove a key isoperimetric inequality in the corresponding Riemannian metric. The Markov chain for a given log-concave distribution is defined in Section~\ref{sec:mixing}. Conditions on the size of a step are introduced in Section~\ref{sec:step_size_cond}, and a lower bound on the conductance of the chain is proved in Section~\ref{sec:cond}. Section~\ref{sec:tracking} contains main results about tracking time-varying distributions given appropriate measures of change between time steps. Section~\ref{sec:apps} is devoted to applications. Finally, Sections~\ref{sec:proofs} and \ref{sec:variation} contain all the remaining proofs.

\section{Geometry Induced by the Self-Concordant Barrier}

The Markov chain studied in this paper uses as a proposal a Gaussian distribution with a covariance that approximates well the local geometry of the set $\K$ at the current point. This local geometry plays a crucial role in the theory of interior point methods for optimization, yet for our purposes a handle on the local geometry yields a good lower bound on \emph{conductance} of the Markov chain. Further intriguing similarities between optimization and sampling will be pointed out throughout the paper. 

We refer to \cite{Nemirovski04lectures} for an introduction to the theory of interior point methods, a subject centered around the notion of a self-concordant barrier. Once we have defined a self-concordant barrier for $\K$, the local geometry is defined through the Hessian of the barrier at the current point. To be more precise, for any function $F$ on the interior $int(\K)$ having continuous derivatives of order $k$, for vectors $h_1, \dots, h_k \in \R^d$ and $x \in int(\K)$, for $k \geq 1$, we
recursively define 
\begin{align*}
	&D^kF(x)[h_1, \dots, h_k] ~\deq~ \lim_{\eps \ra 0 } \frac{D^{k-1} (x + \eps h_k) [h_1, \dots, h_{k-1}] - D^{k-1} (x) [h_1,\dots, h_{k-1}]}{\eps},
\end{align*} 
where $D^0F(x) \deq F(x)$.
Let $F$ be a self-concordant barrier of $\K$ with a parameter $\nu$ (see Appendix~\ref{sec:self_conc_def} for the definition and Section~\ref{sec:apps} for examples). The barrier induces a  Riemannian metric whose metric tensor is the Hessian of $F$ \cite{NesterovTodd08}. In other words, the metric tensor on the tangent space at $x$ assigns to a vector $v$ the length 
$$\|v\|^2_x \deq D^2 F(x)[v, v],$$ 
and to a pair of vectors $v, w$, the inner product $$\inner{v, w}_x \deq D^2F(x)[v, w] \ .$$  The unit ball in $\|\cdot\|_x$ around a point $x$ is called the \emph{Dikin ellipsoid} \cite{Nemirovski04lectures}.  

For $x, y \in \K$, let $\rho(x, y)$ be the Riemannian distance $\rho(x, y) = \inf_\Gamma \int_z \|d
\Gamma\|_z$ where the infimum is taken over all rectifiable paths $\Gamma$ from $x$ to $y$. Let $\MM$ be the metric space whose point set is $\K$ and metric is $\rho$, and define $\rho(S_1,S_2) = \inf\limits_{x \in S_1, y \in S_2} \rho(x, y)$.  The first main ingredient of the analysis is an isoperimetric inequality.
\begin{theorem}\label{thm:LVanalog} Let $S_1$ and $S_2$ be measurable subsets of $\K$ and $\mu$ a probability measure supported on $\K$ that
possesses a density whose logarithm is concave. Then it holds that
\begin{align*} 
	\mu((\K \setminus S_1) \setminus S_2)  \geq \frac{1}{2(1 + 3 \nu)} \rho(S_1, S_2) \mu(S_1) \mu(S_2).
\end{align*} 
\end{theorem}
The theorem ensures that two subsets well-separated in $\rho$ distance must have a large mass between them. A lower bound on conductance of our Markov chain will follow from this isoperimetric inequality. We remark that convexity of the set $\K$ is crucial for the above property. A classical example of a non-convex shape with a ``bottleneck'' is a dumbbell. For this body, the above statement clearly fails, and a ``local'' random walk on such a body gets trapped in either of the two parts for a long time.

\section{The Markov Chain}
\label{sec:mixing}

Let ${\mathcal B}$ be the Borel $\sigma$-field on $\K$. Given an initial probability measure on $\K$, a Markov chain is specified by a collection of one-step transition probabilities $$\{\Ps(x,B), x\in\K, B\in{\mathcal B}\}$$ such that $x\mapsto \Ps(x,B)$ is a measurable map for any $B\in{\mathcal B}$ and $\Ps_x(\cdot)\deq \Ps(x,\cdot)$ is a probability measure on $\K$ for any $x\in\K$. 

For $x \in int(\K)$, let $\g^r_x$ denote
the unique Gaussian probability density function on $\R^d$ such that 
$$\g^r_x(y) \varpropto \exp \left(-\frac{d\|x - y \|_x^2 }{r^2} +
V(x)\right), ~~~~V(x) \deq \frac{1}{2} \ln  \det  D^2F(x)$$ and $r$ is a parameter that is chosen according to a condition specified below. The covariance of this distribution is given by the Hessian of $F$ at point $x$, and thus the contour lines are scaled Dikin ellipsoids.

The Markov chain considered in this paper is based on the Dikin Walk introduced by Kannan and Narayanan \cite{kannan2012random}. Adapted to sampling from log-concave distributions in this paper, the Markov chain is parametrized by a convex function $s$ and a step size $r$. Rather than writing out the unwieldy explicit form of the transition kernel $\Ps_x$, we can give it implicitly as the following random walk:\\

\begin{center}
    \begin{minipage}{.75\columnwidth} \tt
    		\begin{itemize}
			\item[]\hspace{-1cm}  With probability $1/2$, set $w := x$. 
			\item[]\hspace{-1cm}  With probability $1/2$, sample $z$ from $\g^r_{x}$ and
				\begin{enumerate}
					\item[]\hspace{-1cm} If $z \notin \K$, let $w := x$. 
					\item[]\hspace{-1cm} If $ z \in \K$, let $w :=
			                    \begin{cases}
			                      z & \text{with prob. } \min\left(1,\,\, \frac{\g^r_z(x) \exp(s(x))}{\g^r_{x}(z)\exp(s(z))}  \right) \\
			                      x & \hbox{otherwise.}
			                    \end{cases}$
				\end{enumerate}
			\end{itemize}
			
	\vspace{-0mm}
    \end{minipage}
\end{center}
% \begin{algorithm}[H]
% 	\caption{One Step of Markov Chain~~ $\step(x, s, r)$}
% 	\begin{enumerate}
% 	\item[]\hspace{-1cm} \textbf{Input}: current point $x\in\K$, convex function $s$, step size $r$.
% 	\item[]\hspace{-1cm} \textbf{Output}: next point $w \in \K$
% 	\item[]\hspace{-1cm}  Toss a fair coin. \textbf{If} \verb"Heads", set $w := x$.
% 	\item[]\hspace{-1cm}  \textbf{Else},
% 		\begin{enumerate}
% 			\item[]\hspace{-1cm} Sample $z$ from $\g^r_{x}$. If $z \notin \K$, let $w := x$.
% 			\item[]\hspace{-1cm} If $ z \in \K$, let \\
% 				$w :=
% 	                    \begin{cases}
% 	                      z & \text{with prob. } \min\left(1,\,\, \frac{\g^r_Z(x) \exp(s(x))}{\g^r_{x}(z)\exp(s(z))}  \right) \\
% 	                      x & \hbox{otherwise.}
% 	                    \end{cases}$
% 		\end{enumerate}
% 	\end{enumerate}
% 	\label{algo:one_step}
% \end{algorithm}
The Markov chain is \emph{lazy}, as it stays at the current point with probability at least $1/2$. This ensures uniqueness of the stationary distribution \cite{LovSim93}. Furthermore, a simple calculation shows that the detailed balance conditions are satisfied with respect to a stationary distribution $\mu$ whose density (with respect to the Lebesgue measure) is proportional to $\exp(-s(x))$. Indeed, to see that $\mu(x)\Ps_x(dz) = \mu(z)\Ps_z(dx)$, it suffices to observe that
\begin{align*}
	\exp(-s(x))\g^r_x(z) &\min\left(1,\,\, \frac{\g^r_z(x) \exp(s(x))}{\g^r_{x}(z)\exp(s(z))}  \right)\\
	&~~~~~~~~~~~~~~~~~~~~=\exp(-s(z))\g^r_z(x) \min\left(1,\,\, \frac{\g^r_x(z) \exp(s(z))}{\g^r_{z}(x)\exp(s(x))}  \right).
\end{align*}
Therefore the Markov chain is reversible and has the desired stationary measure $\mu$.  

The value of $r$ has a specific meaning: most of the $y$'s sampled from $\g^r_x$ are within a thin ``Dikin shell'' of radius proportional to $(\E \|x-y\|^2_x)^{1/2} = r$ by measure-concentration arguments. We will therefore refer to $r$ as the effective ``step size''. An important and non-trivial result from the theory of interior point methods is that the unit Dikin ellipsoid is contained in the set $\K$ and gives a good approximation to the local geometry of the set (see Figure~\ref{fig:dikin_walk} below). Thanks of this fact, the sampling procedure has in general better mixing properties than the Ball Walk \cite{LovSim93,Vempala05survey}.

\subsection{\textbf{Step Size Conditions}}
\label{sec:step_size_cond}

The analysis of the Markov chain requires the steps $r$ to be not too large to ensure that different enough transition probability functions happen only for far away points. The precise upper bounds on $r$ depend on the convex function $s(x)$ and can be calculated on the fly when we move to the setting of a time-varying function. We give four conditions:

\begin{condition}[Linear Functions]
	\label{cond:lin}
	If $s$ is linear, we may set $r = 1/d$. 
\end{condition}
\begin{condition}[Lipschitz Functions]
	\label{cond:lip}
	For a function $s$ that is $L$-Lipschitz with respect to the Euclidean norm, we may set the step size $r = \min\left\{\frac{1}{d}, \frac{1}{L}\right\}$.
\end{condition}
\begin{condition}[Smooth Functions]
	\label{cond:smooth}
	Suppose $s$ has Lipschitz-continuous gradients: there exists $\sigma>0$ such that $\|\nabla s(x)- \nabla s(y)\|\leq \sigma\|x-y\|$. We may then set the step size to be $\min\left\{\frac{1}{d}, \frac{1}{\sqrt{\sigma}}\right\}$.
\end{condition}
These three conditions can be shown to follow from a more general sufficient step size condition that is based on ``local'' information:
\begin{condition}[General Condition]
	\label{cond}
	Fix constants $C,C'>0$. Given the convex function $s(x)$, the step size
	$r \leq \min\left\{
		\frac{1}{d}, r^*
		\right\}$
	is a valid choice if there exists a linear function $<g,x>$ such that
	$$r^* \leq \sup \left\{~r~:~ \forall z, w\in\K ~\text{ with }~ \|z-w\|_z \leq C'r, ~~~ \Big| s(z)- s(w)-\inner{g,z-w} \Big| < C \right\} \ .$$
\end{condition}
The condition says that for two points, with one being inside the $O(r)$-Dikin ellipsoid around the other point, the function is within a constant of  being linear. It follows from the last condition that, for instance, if $s(x)=\inner{b,x}+a(x)$ is a sum of a linear and a non-linear Lipschitz part, the step size is only affected by the Lipschitz constant of the non-linear part.

It is simple to verify that the step size in Condition~\ref{cond:lip} satisfies Condition~\ref{cond}. Indeed, for any $w$ such that $\|z-w\|_z\leq C'r$, we have $\|z-w\|\leq C''r R$ (where $R$ is the radius of the largest ball contained in $\K$). Take $g_z$ and $g_w$ to be any subgradients of $s$ at $z$ and $w$, respectively. We then have
$$\left| s(z)-s(w)-\inner{g_w,z-w} \right| \leq \inner{g_z-g_w, z-w} \leq 2L\|z-w\| \leq 2 \ .$$
Notice that for Condition~\ref{cond:smooth}, the above calculation becomes 
$$ \inner{g_z-g_w, z-w} \leq \sigma\|z-w\|^2 \leq 1 \ .$$

In the remainder of this paper, $C$ will denote a universal constant that may change from line to line. The exact value of the final constant in Lemma~\ref{1lcond} below can be traced in the proofs; we omit this calculation for the sake of brevity.

\subsection{\textbf{Conductance of the Markov Chain}}
\label{sec:cond}

In order to show rapid mixing of the proposed Markov chain, we prove a lower bound on its \emph{conductance}
$$\Phi \deq \inf\limits_{\mu(S_1) \leq \frac{1}{2}}\frac{\int_{S_1} \Ps_x(\K\setminus S_1) d\mu(x)}{\mu( S_1)},$$
where $\Ps_x$ is the one-step transition function defined earlier. Once such a lower bound is established, the following general result on the reduction of distance between distributions will imply exponentially fast convergence.
\begin{theorem}[Lov\'{a}sz-Simonovits \cite{LovSim93}]\label{1thm:L1} Let $\gamma_0$ be the initial distribution for a lazy reversible ergodic Markov chain whose conductance is $\Phi$ and stationary measure is $\gamma$.
For every bounded $f$, let $\|f\|_{\gamma} \deq \sqrt{\int_\K f(x)^2 d\gamma(x)}$. For any fixed $f$, let $Ef$ be the map that takes $x$ to $\int_\K f(y) d\Ps_x(y)$. Then if $\int_\K f(x) d\gamma(x) = 0$, it holds that
$$\|E^k f\|_{\gamma} \leq \left(1 -\frac{\Phi^2}{2}\right)^k \|f\|_{\gamma} \ .$$
\end{theorem}

To prove a lower bound on conductance $\Phi$, we first relate the Riemannian metric $\rho$ to the proposed Markov Chain. Intuitively, the following result says that for close-by points, their transition distributions cannot be far apart in the total variation distance $d_{TV}$.
\begin{lemma}\label{1lem:gp}
If $x, y \in \K $ and $\rho(x, y) \leq \frac{ r}{C \sqrt{d}}$
for some constant $C$, then $$d_{TV}(\Ps_x, \Ps_y) \leq 1 - \frac{1}{C'}$$ 
for some constant $C'$.
\end{lemma}
Lemma~\ref{1lem:gp} together with the isoperimetric inequality of Theorem~\ref{thm:LVanalog} give a lower bound on conductance of the Markov Chain. 
\begin{lemma}\label{1lcond}
    Let $\mu$ be a log-concave distribution with support on $\K$ whose density with respect to the Lebesgue measure is proportional to $\exp\{-s(x)\}$, and suppose an appropriate step size condition (Section~\ref{sec:step_size_cond}) for the Markov chain is satisfied. Then there exists a constant $C>0$ such that the
    conductance of the above Markov chain is bounded below as
	$$\Phi \geq \frac{r}{C \nu \sqrt{d}} \ .$$ 
\end{lemma}
We remark that the step size $r$ enters the lower bound on $\Phi$. While we would like the steps to be large, the conditions outlined earlier dictate a limitation on how large $r$ can be. In particular, we always have $r\leq 1/d$. The step size needs to be even smaller for functions $s$ for which a linear approximation is poor.

\section{Tracking the Distributions}
\label{sec:tracking}

Having specified the Markov chain and the step size, we now turn to the problem of tracking a sequence of distributions $\mu_1,\ldots,\mu_t,\ldots$. For each $t\geq 1$, define a Markov chain with parameters $r_t$ and $s_t$, and let its transition kernel be denoted by $\Ps_{t}(x,B)$ for $x\in\K$ and $B\in{\mathcal B}$. Let $\Phi_t$ denote the conductance of this chain. The  chain will be run for $\tau_t$ steps starting from the end of the chain at time $t-1$. Formally, let the $i$-th step of the $t$-th chain be denoted by the random variable $X_{t,i}$. Define $\tau_0=0$ and let $\sig{0}{0}$ be the initial distribution of $X_{0,0}$. Then $X_{t,i}$ has distribution
$$\sig{0}{0}\Ps_{1}^{\tau_1}\cdots \Ps_{t-1}^{\tau_{t-1}} \Ps_t^{i}$$
and we have made the identification $X_{s,\tau_s} = X_{s+1,0}$, gluing the successive chains together. Let the  distribution of $X_{t,i}$ be denoted by $\sig{t}{i}$. By the definition of the chain, $\sig{t}{i}$ is a distribution with bounded density, supported on $\K$.

\begin{figure}[htbp]
	\centering
		\includegraphics[height=1.2in]{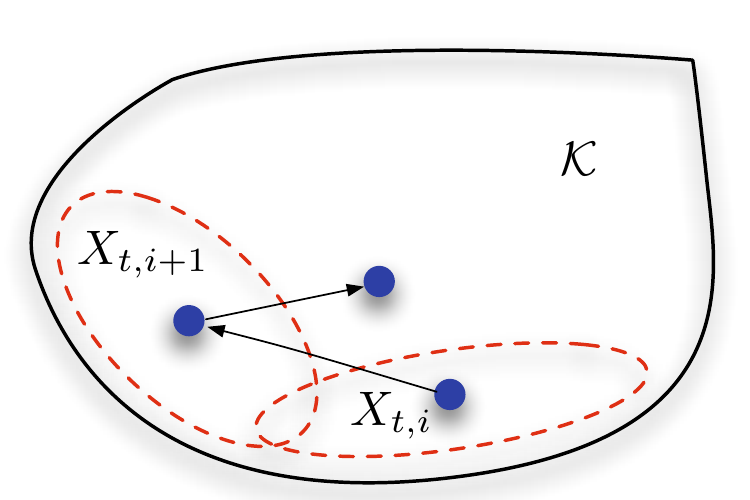}
	\caption{Steps of the Dikin Walk. The next point is sampled from a Gaussian distribution with a shape (contours depicted with dashed lines) corresponding to Dikin ellipsoids. These ellipsoids approximate well the local geometry.}
	\label{fig:dikin_walk}
\end{figure}

\subsection{\textbf{Measuring the Change}}
Let $\| \cdot\|_{t}$ denote the $\LL_2$ norm with respect to the measure $\mu_t$, defined as $\|f\|_{t} = \left(\int_\K {f}^2 d \mu_t\right)^{1/2}$ for a measurable function $f : \K \ra \R$.  Further, let $\|\cdot\|_\K$ denote the supremum norm $\|f\|_\K = \sup_{x\in\K} |f(x)|$ and let
\begin{align}
	\label{eq:bounded_ratio}
	\beta_{t+1} = \max\left\{ \left\|d\mu_t / d\mu_{t+1}  \right\|_\K, \left\|d\mu_{t+1} / d\mu_{t}  \right\|_\K \right\} \ .
\end{align}
This ratio provides an upper bound on the point-wise change of the density function. 
A straightforward way to upper bound $\beta_{t+1}$ is by writing 
$$\sup_{x\in\K} \frac{e^{-s_t(x)} }{e^{-s_{t+1} (x) }}\frac{\int_{\K}e^{-s_{t+1}(x)}dx}{\int_{\K}e^{-s_{t}(x)}dx} \leq \sup_{x\in\K} e^{2|s_t(x)-s_{t+1}(x)|} 
$$
and, hence, 
\begin{align}
	\label{eq:log_beta}
	\log \beta_{t+1} \leq ~2\| s_t(x)-s_{t+1}(x)\|_\K \ .
\end{align}
Another way to measure the change in successive distributions is with respect to the $\mathcal{L}_2$ norm:
\begin{align}
	\label{eq:l2_bounded_ratio}
	\alpha_{t+1} = \left\|d\mu_t/d\mu_{t+1}  \right\|_{t+1} \ .
\end{align}
In contrast to the point-wise change, the ratio $\alpha_{t+1}$ is more difficult to calculate. In this respect, the following result, which follows from the proof of \cite{LovVem06simulated,kalai2006simulated}, will be useful:
\begin{lemma}
	\label{lem:l2change_bdd} Let $s_t$ be a convex function and $s_{t+1}=\left(1-\delta\right)^{-1} s_t$. Let $\mu_t$ and $\mu_{t+1}$ be defined as in \eqref{eq:def_mu_t}. Then
	$$\alpha_{t+1} \leq \left( 1+ \frac{\delta^2}{1-2\delta}\right)^{d/2} $$
	In particular, if $\delta\leq d^{-1/2} \leq 1/3$, then $\alpha_{t+1}\leq 5$.
\end{lemma}
We remark that the ratio between $\mu_t$ and $\mu_{t+1}$ measured in the supremum norm may be exponentially large, while the $\mathcal{L}_2$ change is small. As in \cite{LovVem06simulated, kalai2006simulated}, this fact will be crucial in this paper when we study simulated annealing.

\subsection{\textbf{Tracking the Distributions: Main Results}}
Denote the error in approximating the stationary distribution at the end of $t$-th chain by
\begin{align}
	\xi_t ~\deq~ \left\|\frac{d\sig{t}{\tau_t}}{d\mu_t} - 1 \right\|_t
\end{align}
and let $$\Delta_t ~\deq~ \frac{r_{t}^2}{C d \nu^2} \ .$$ 

\begin{theorem}
	\label{thm:recurrence}
	The errors $\xi_t$ satisfy the recurrence
	\begin{align}
		\label{eq:upper_error}
		%\xi_{t} \leq (1-\Delta_{t})^{\tau_{t}}(\sqrt{\beta_t} \xi_{t-1} + \sqrt{\beta_t-1}) 
		\xi_{t} \leq (1-\Delta_{t})^{\tau_{t}}(\beta_t^{3/2} \xi_{t-1} + \sqrt{\beta_t}(\beta_t-1)) 
	\end{align}
	for any $t\geq 1$.
\end{theorem}
\begin{proof}[Proof of Theorem~\ref{thm:recurrence}]
	We iteratively apply Theorem~\ref{1thm:L1} with $f = \frac{d\sig{t}{j}}{d\mu_t}-1$ and the stationary distribution $\gamma=\mu_t$, and observe that $Ef$ takes $\sig{t}{j}$ to $\sig{t}{j+1}$. Then from Lemma~\ref{1lcond}, for $t\geq 1$ and $i\geq 1$,
		\begin{align*}
			\norm{\frac{d\sig{t}{i}}{d\mu_t} - 1 }_{t}  \leq \norm{\frac{d\sig{t}{0}}{d\mu_t} - 1 }_{t} \cdot \left(1 - \Delta_t\right)^i
		\end{align*}
		Using the first part of Lemma~\ref{lem:hstep2} (see Section~\ref{sec:proofs})
		\begin{align*}
			\norm{\frac{d\sig{t}{0}}{d \mu_{t}} - 1 }_{t}  \leq \beta_{t}^{3/2} \norm{ \frac{d\sig{t}{0}}{d\mu_{t-1}}  - 1 }_{t-1} + \sqrt{\beta_{t}}(\beta_{t}-1) ,
		\end{align*}
		concluding the proof. An alternative recurrence, using the second part of Lemma~\ref{lem:hstep2}, is
		$$\xi_{t} \leq (1-\Delta_{t})^{\tau_{t}}(\sqrt{\beta_t} \xi_{t-1} + \sqrt{\beta_t-1}), $$
		which is better for large $\beta_t$ but worse for $\beta_t\approx 1$.
\end{proof}

We would like to adaptively choose $\tau_t$ to make the right-hand side \eqref{eq:upper_error} small. While the value of the error $\xi_{t-1}$ at the previous round is not available for this purpose, let us maintain an upper bound $u_{t-1}$ on this error. Thus, we may write $\tau_t$ as a function $\tau_t(u_{t-1}, s_t, r_t, \beta_t)$. Suppose at round $t=0$ we ensure that $\xi_0\leq u_0$. Then, recursively, we may compute $u_t$ as the upper bound in \eqref{eq:upper_error}:
\begin{align}
	\label{eq:recur_u_t}
	u_t \geq (1-\Delta_{t})^{\tau_{t}}(\beta_t^{3/2} u_{t-1} + \sqrt{\beta_t}(\beta_t-1))
\end{align}
Then, given the initial condition, we have $\xi_t\leq u_t$ for all $t\geq 0$. 

Let us consider some consequences of Theorem~\ref{thm:recurrence}. In particular, we are interested in situations when we can track the distributions with only \emph{one step} of the random walk.
\begin{corollary} 
	\label{cor:singlestep}
	Let $\tau_t = 1$ for all $t\geq 1$ and suppose $\xi_0 \leq u_0 = \sqrt{\beta_0}(\beta_0-1)/\Delta_0$ with $\Delta_0=\frac{1}{C d^3 \nu^2} \leq \frac{1}{2}$. Assume that $\beta_t$ is non-decreasing and $\Delta_t$ is non-increasing in $t$, and suppose 
	\begin{align}
		\label{eq:cor_relation_one_step}
		\beta_t^{3/2} \leq 1+ \frac{\Delta_t^2}{1-\Delta_t}
	\end{align}
	for all $t\geq 1$. Then we have 
	$$ \xi_t \leq u_t = \frac{\sqrt{\beta_t}(\beta_t-1)}{\Delta_t}$$
	for all $t\geq 0$. In particular, \eqref{eq:cor_relation_one_step} is satisfied whenever $\beta_t-1 \leq 0.4\Delta_t^2$.
\end{corollary}
The proof of the above corollary follows from the more general result:
\begin{corollary}
	\label{cor:multistep}
	Fix a sequence $\epsilon_0,\ldots,\epsilon_t,\ldots $ of positive target accuracies and assume $\xi_0\leq \epsilon_0$. It is then enough to set
	\begin{align}
		\tau_t = \left\lceil\frac{1}{\Delta_t}\log\left( \beta_t^{3/2}\cdot\frac{\epsilon_{t-1}}{\epsilon_t}+\frac{\sqrt{\beta_t}(\beta_t-1)}{\epsilon_t}\right) \right\rceil
	\end{align}
	in order to ensure 
	$\xi_t \leq \epsilon_t$ for each $t\geq 0$. 
\end{corollary}
\begin{proof}
	Immediate by writing
	$$u_t = (1-\Delta_{t})^{\tau_{t}}(\beta_t^{3/2} \epsilon_{t-1} + \sqrt{\beta_t}(\beta_t-1)) \leq \epsilon_t,$$
 solving for $\tau_t$, and using the approximation
	$\log(1/(1-\Delta_t)) \geq \log(1+\Delta_t)\geq \Delta_t \ .$
\end{proof}

We now consider the case when one has control on the $\mathcal{L}_2$ norm $\alpha_{t}$ of the change between successive distributions.
First, observe that closeness of the distributions in the norm $\|\cdot\|_t$ implies closeness in total variation distance as
\begin{align}
	\label{eq:l2_implies_tv}
	\int |d\sig{t}{i} - d\mu_t| = \int \left|\frac{d\sig{t}{i}}{d\mu_t} -1\right|d\mu_t 
	%\leq \left(\int \left(\frac{d\sig{t}{i}}{d\mu_t} -1\right)^2 d\mu_t \right)^{1/2} = 
	\leq \left\|\frac{d \sig{t}{i}}{d\mu_{t}} - 1  \right\|_{t} \ .
\end{align}

\begin{proposition}
	\label{prop:l2_multistep}
	Fix a sequence $\epsilon_0,\ldots,\epsilon_t,\ldots $ of positive target accuracies and assume $d_{TV} (\sig{0}{0}, \mu_0)\leq \epsilon_0$. Suppose we set
	\begin{align}
		\tau_t = \left\lceil\frac{1}{\Delta_t}\log \left(\frac{\alpha_t}{\epsilon_t}\right) \right\rceil \ .
	\end{align}
	Then the total variation distance between $\sig{t}{\tau_t}$ and $\mu_t$ is bounded as
	\begin{align}
		\label{eq:tv_bound}
		d_{TV} (\sig{t}{\tau_t}, \mu_t) \leq \sum_{s=0}^t \epsilon_s
	\end{align}
	for each $t\geq 0$. 
\end{proposition}
\begin{proof}
	For any $t\geq 1$, let us write 
	\begin{align}
		\label{eq:decomp}
		\sig{t}{\tau_t} = \mu_{t} + \gamma_t
	\end{align} with a signed measure $\gamma_t = \sig{t}{\tau_t} - \mu_{t}$. By way of induction, suppose \eqref{eq:tv_bound} holds for time $t$. Consider the operator $E_{t+1}$ corresponding to the random walk of the $t+1$-st chain. The operator acts on a function $f$ by taking $f$ to $\int_\K f(y) d\Ps_{t+1}(x,y)$. Then applying Theorem~\ref{1thm:L1} to the function $d\mu_{t}/d\mu_{t+1}-1$, we have
	$$\left\|E_{t+1}^{\tau_{t+1}} \left(\frac{d\mu_{t}}{d\mu_{t+1}}-1 \right) \right\|_{t+1}\leq \epsilon_{t+1}$$
	by the choice of $\tau_{t+1}$ and the definition of $\alpha_{t+1}$. That is, upon the action of $E_{t+1}^{\tau_{t+1}}$,  $\mu_t$ is mapped to $\mu_{t+1}$ within an error of at most $\epsilon_{t+1}$ in the $\mathcal{L}_2$ sense (and, hence, in the total variation sense). Since the operator $E_{t+1}$ is non-expanding in the $\mathcal{L}_1$ sense, total variation of $\gamma_t$ does not increase under the action of $E_{t+1}^{\tau_{t+1}}$. In view of the inductive hypothesis for step $t$, we conclude $d_{TV} (\sig{t+1}{\tau_{t+1}}, \mu_{t+1}) \leq \sum_{s=0}^t \epsilon_s + \epsilon_{t+1}$, as desired.
\end{proof}

\section{Applications}
\label{sec:apps}

Before diving into the applications of the random walk, let us give several examples of sets $\K$ for which the self-concordant barrier $F$ and its Hessian can be easily calculated. In the following examples, assume that $\K$ has non-empty interior.
\begin{example} Suppose $\K$ is given by $m$ linear constraints of the form $\inner{a_j,x}\leq b_j$, $j=1,\ldots,m$. Then $F(x) = -\sum_{j=1}^m \log (b_j-\inner{a_j,x})$ is a self-concordant barrier with parameter $\nu=m$. The Hessian is easily computable:
	$$D^2 F(x) = \sum_{j=1}^m \frac{a_j a_j^\tr}{(b_j-\inner{a_j,x})^2} \ .$$
\end{example}
\begin{example}
	Let $\K = \{x\in\reals^d: f_j(x)\leq 0, j=1,\ldots,m\}$ where each $f_j$ is a convex quadratic form. Then $F(x) = -\sum_{j=1}^m \log(-f_j(x))$ is a self-concordant barrier with parameter $m$. As an example, the function $-\log(R-\|x\|^2)$ is a self-concordant barrier for the unit Euclidean sphere $\{x: \|x\|^2-1\leq 0\}$, with parameter $\nu=1$, and the Hessian is given by
	$$D^2 F(x) = \frac{2}{1-\|x\|^2}I + \frac{4}{(1-\|x\|^2)^2} xx^\tr \ .$$
\end{example}
Importantly, there always exists a self-concordant barrier with $\nu=O(d)$; yet, for some convex sets (such as the sphere) the parameter can even be constant.

Self-concordant barriers can be combined: if $F_j$ is $\nu_j$-self-concordant for $\K_j$, $j=1,\ldots,m$, then $\sum_j F_j$ is $\sum_j \nu_j$-self-concordant for the intersection $\cap_i \K_i$, given that it has nonempty interior. Thus, closed forms for the Hessian of the barrier, required for defining $\g^r_x$ in our Markov chain, can be calculated for many sets $\K$ of interest.  We refer to \cite{Nemirovski04lectures,NNbook} for further powerful methods for constructing the barriers.

\subsection{\textbf{Sampling from Posterior in Exponential Families}}
Suppose data $y_1,y_2,\ldots \in \Y$ are distributed i.i.d. according to a member of an exponential family with natural parameter $x$:
$$p(y|x) = \exp\{\inner{x,T(y)} - A(x) \} h(y)$$
where $A(x) = \int h(y)\exp\left\{\inner{x,T(y)}\right\}$ is a convex function and $T:\Y \mapsto \reals^d$ is a sufficient statistic. Suppose $x\in \K$; that is, we have some knowledge about the support of the parameter. We have in mind the situation where data arrive one at a time and we are interested in sampling from the associated posterior distributions.  The likelihood function after seeing $y_1,\ldots,y_t$ is
$$\ell(x) \propto \exp\left\{\inner{x,\sum_{i=1}^t T(y_i)} - tA(x)\right\}$$
and, together with a conjugate prior 
$\pi_{\kappa_1,\kappa_2}(x) \propto \exp\left\{\inner{x,\kappa_1}-\kappa_2 A(x)\right\}$ for some $(\kappa_1,\kappa_2)\in\reals^{d+1}$, 
we obtain the posterior distribution at time $t$
$$p_t(x|y) \propto \exp\left\{ \inner{x,\kappa_1+\sum_{i=1}^t T(y_i)} - (t+\kappa_2)A(x)\right\} \ .$$
We apply the sampling technique to this scenario by defining
$$s_0(x) = -\inner{x,\kappa_1} + \kappa_2 A(x), ~~~~ s_t(x)=-\inner{x,\kappa_1+\sum_{i=1}^t T(y_i)} + (t+\kappa_2)A(x) \ .$$
It remains to calculate the number of steps required to track the distributions as additional data arrive one-by-one. Let $L$ be the Lipschitz constant of $A(x)$ over $\K$ with respect to Euclidean norm, and let us assume $L$ to be finite. Then Condition~\ref{cond} is satisfied with $r=\min\left\{\frac{1}{(t+\kappa_2)L},\frac{1}{d}\right\}$. Furthermore, we may set 
$$\beta_t = \sup_{x\in\K}\exp\left\{2|\inner{x,T(y_{t})}-A(x)|\right\},$$
a quantity that depends on the observed data. Importantly, we do not need to provide an a priori data-independent bound of this type, which might not be finite. 

Suppose we would like to maintain a constant level $\epsilon>0$ of accuracy at each step $t$. Corollary~\ref{cor:multistep} guarantees this accuracy if each chain is run for 
$$\tau_t = \left\lceil \frac{1}{\Delta_t}\log\left( \beta_t^{3/2}+\frac{\sqrt{\beta_t}(\beta_t-1)}{\epsilon}\right) \right\rceil= \mathcal{O}\left(d\nu^2\max\{(t+\kappa_2)^2L^2, d^2\}+ \log(1/\epsilon)\right) \ .$$
One of the features of this bound is a relatively benign dependence on the dimension $d$, especially if the geometry of the set $\K$ allows the parameter $\nu=\mathcal{O}(1)$, as in the case of a sphere. On the negative side, the number of steps needed after seeing $t$ data points is proportional to $t^2$. Such an adverse dependence, however, is to be expected as the posterior distribution becomes concentrated very quickly.

We now demonstrate that stronger results can be achieved under additional assumptions via Condition~\ref{cond:smooth}. Suppose that $A$ is smooth: there exists $H\succeq 0$ such that
$$A(x)\leq A(w)+\inner{\nabla A(x), w-x} + (w-x)^\tr H (w-x)$$
for any $w,x\in \K$. This is a natural assumption, as the second derivative of the log normalization function $A$ corresponds to the variance of the random variable with the given parameter; furthermore, $A$ is differentiable. Let $\lambda_{\max}$ be the largest eigenvalue of $H$.
Then the condition yields $r_t=\frac{C}{\sqrt{(t+\kappa_2)\lambda_{\text{max}}}}$. To obtain $\epsilon$-accuracy, it suffices to set
$$\tau_t = \mathcal{O}\left(d\nu^2\max\{(t+\kappa_2)\lambda_{\max}, d^2\}+ \log(1/\epsilon)\right),$$
which has only linear dependence on the size of the data seen so far.

We remark that each step of the random walk requires evaluation of the log-partition function $A(x)$. If this function is not available in closed form, we may approximate the value $A(x)$ for each query $x$. In order to do this, we may run an additional sampling procedure with $s'(x)=\inner{x, T(y)}$. Alternatively, we may appeal to known methods for this problem, such as Hit-and-Run \cite{Vempala05survey}.

\subsection{\textbf{Sampling from Drifting Truncated Distributions}}

In the previous example, we employed the Markov chain to sample a parameter from a log-concave posterior. We now turn to the question of sampling from a log-concave distribution restricted to a convex set. This problem has a long history (see e.g. \cite{devroye1986book, gilks1992adaptive}), and it is recognized that sampling from truncated distributions is difficult even for nice forms such as the Normal distribution. One successful approach to this problem is the Gibbs sampling method \cite{robert1995simulation,damien2001sampling}, yet the rate of convergence is not generally available. The MCMC method of this paper yields a provably fast algorithm for such situations. Furthermore, we can track a drifting distribution over $\K$ with a small number of steps. 

For illustration purposes, we study a simple example of a truncated Normal distribution; the same techniques, however, apply more generally. To simplify calculations, suppose the distributions $\mu_t$ are defined to be ${\cN}({\bf c}_t, \frac{1}{d}I)$ over a convex compact set $\K \subset \reals^d$ and suppose the mean ${\bf c}_t$ is drifting within a Euclidean ball of radius $R$. With the definition in \eqref{eq:def_mu_t} we have $s_t(x) = \frac{1}{2}\|x-{\bf c}_t\|^2$. Define the drift $\delta_{t} = \|{\bf c}_t - {\bf c}_{t-1}\|$. In view of \eqref{eq:log_beta},
$$\log \beta_{t} \leq \sup_{x\in\K} \|{\bf c}_t-{\bf c}_{t-1}\|\cdot \|2x-{\bf c}_t-{\bf c}_{t-1}\|\leq C_{R,\K} \delta_{t}$$
where $C_{R,\K}$ depends on the radius $R$ and the radius of a smallest Euclidean ball enclosing $\K$. In the same manner, the  Lipschitz constant of $s_t(x)$ over $\K$ can be upper bounded by $L_{R,\K}$ that depends solely on the two radii. We may thus set the step size to be $r_t=\min\{\frac{1}{d},\frac{1}{L_{R,\K}}\}$. If we aim for a fixed target accuracy $\epsilon$ for all $t$, by Corollary~\ref{cor:multistep}, it is enough to make
\begin{align}
	\label{eq:multistep_suff_tau}
	\tau_t = \left\lceil \frac{1}{\Delta_t}\log\left( \beta_t^{3/2}+\frac{\sqrt{\beta_t}(\beta_t-1)}{\epsilon}\right) \right\rceil
\end{align}
steps. In the case that the drift $\delta_t$ is small enough, only one step is sufficient. To quantify the regime when this happens, observe that $\beta_t \leq \exp\{C_{R,\K}\delta\} \leq 1+C\delta_t$, and it is then enough to require 
$$\delta_t = \mathcal{O}\left(\Delta_t^2\right) =  \mathcal{O}\left(\frac{\min\{1/d^2, 1/L_{R,\K}^2\}}{d\nu^2}\right)$$
 in view of \eqref{eq:cor_relation_one_step}. It is quite remarkable that the one-step random walk can track the changing distribution up to the accuracy $\mathcal{O}\left(\delta_t\frac{d\nu^2}{r_t^2}\right)$, proportional to the size of the drift. Of course, better accuracy can be achieved by performing more steps, as per Corollary~\ref{cor:multistep}.

Another related application is to modeling with mixtures of log-concave distributions. Such models have been successful in clustering \cite{mclachlan2000finite, walther2009inference}, with a mixture of normal distributions being a classical example \cite{fraley2002model}. A mixture of parametric log-concave distributions can be written as $\sum_{i=1}^k \alpha_i \pi_i(\theta_i; x)$; here $\alpha_i$ are positive mixing weights summing to one, and $\pi_i$ are a distributions on $\K$ parametrized by $\theta_i$. A classical method for fitting models to data is the EM algorithm. Given that the parameters $\{\theta_i\}_{i=1}^k$ and the mixing weights $\{\alpha_i\}_{i=1}^k$ have been estimated from data, one may require random samples from this model for integration or other purposes. Given our procedure for sampling from a single log-concave distribution, one may simply pick the mixture according to the weights $\alpha_i$ and then sample from the component. The situation becomes interesting in the case of online arrival of data, when we need to re-compute the EM solution in light of additional data. By the arguments of \cite{RakCap06nips, CapRak06jmlr}, the solution to clustering problems (the analysis was performed for square loss) is \emph{stable} in the following sense: addition of $o(\sqrt{n})$ new data to a sample of size $n$ is unlikely to drastically move the solution (the argument is based on uniqueness of the maximum of an empirical process). This in turn implies that the parameters $\{\theta_i\}$ are unlikely to change by a large amount, and we may thus use the method of sampling from a drifting distribution described earlier. We also remark that the method can be easily parallelized since the Markov chains for the $k$ components do not interact.

\subsection{\textbf{Simulated Annealing for Convex Optimization}}
Let $f(x)$ be a proper convex $1$-Lipschitz function. The aim of convex optimization is to find $\tilde{x}$ with the property $f(\tilde{x}) - \min_{x\in\K} f(x) \leq \epsilon$ for a given target accuracy $\epsilon>0$. We consider the special case of linear function $f(x)=\inner{\ell,x}$, known as Linear Optimization. Complexity of an optimization procedure is often measured in terms of \emph{oracle calls} -- queries about the unknown function. A query about the function value is known as the zero-th order information, while a query about a subgradient at a point -- as the first order information. In the case that the oracle answer is given without noise, it is known that the complexity scales as  $\mathcal{O}\left(\text{poly}(d, \log(1/\epsilon))\right)$. The state-of-the-art result here is the method of \cite{kalai2006simulated, LovVem06simulated} which attains the $d^{4.5}$ dependence on the dimension.

We now apply our machinery to obtain a  $\mathcal{O}\left(\nu^2 d^{3.5}\log(1/\epsilon)\right)$ method. In particular, this yields an improved $d^{3.5}$ dependence on the dimension for the case when $\K$ has a favorable geometry: there exists a self-concordant barrier with a parameter $\nu=\mathcal{O}(1)$.

We use the annealing scheme of \cite{kalai2006simulated}. To this end, we set $s_t = \left(1-d^{-1/2}\right)^{-t} f$ and observe that the assumption of Lemma~\ref{lem:l2change_bdd} is satisfied with $\delta = d^{-1/2}$. Since functions are linear, we may set the step size $r_t = 1/d$ for all $t$. Hence, $\alpha_{t} \leq 5$ whenever $d>8$ (and a different constant can be obtained for smaller $d$ from the proof). By Proposition~\ref{prop:l2_multistep} with a constant accuracy $\epsilon_t= \epsilon\cdot  (\sqrt{d}\log(d/\epsilon))^{-1}$, by making 
\begin{align}
	\tau_t = \left\lceil Cd^3\nu^2\log \left(\frac{5\sqrt{d}\log(d/\epsilon)}{\epsilon}\right) \right\rceil 
\end{align}
steps for $t=1,\ldots, k$, we guarantee
\begin{align}
	\label{eq:opt_total_var}
	d_{TV} (\sig{k}{\tau_{k}}, \mu_{k}) \leq k\epsilon(\sqrt{d}\log(d/\epsilon))^{-1} \ .
\end{align}
According to \cite[Lemma 4.1]{kalai2006simulated}, if $X$ is chosen from a distribution with density proportional to $\exp\{-T^{-1}\inner{\ell, x}\}$, with $\|\ell\|=1$ and some temperature $T>0$, then $$\E(\inner{\ell, X})-\min_{x\in\K} \inner{\ell, x} \leq d T.$$
Hence, we take the desired temperature to be $T=\epsilon/d$, and the number of chains that permits the annealing schedule to reach this temperature can be calculated as $k=\sqrt{d}\log(\frac{d}{\epsilon})$. In view of \eqref{eq:opt_total_var}, the final output of the procedure is an $\epsilon$-accurate solution to the optimization problem. The complexity of the method is then $\mathcal{O}(d^{3.5}\nu^2 \log^2(d/\epsilon))$.

This result can be extended to Lipschitz convex functions beyond linear optimization. However, the step size condition for convex Lipschitz functions requires the steps to be $\mathcal{O}(1/\epsilon)$ towards the end of the annealing schedule. This in turn implies only a suboptimal  $\tilde{\mathcal{O}}(d\nu^2/\epsilon^2)$ complexity. It is an open question of whether Dikin Walk can handle such annealing schedules in a more graceful manner.

\subsection{\textbf{Sequential Prediction}}
Another application of the proposed sampling technique is to the problem of \emph{sequential prediction} with convex cost functions. Within this setting, the learner (or, the Statistician) is tasked with making a series of predictions while observing  a sequence of outcomes on which we place no distributional assumptions. The goal of the learner is to incur cost comparable to that of a fixed strategy chosen in hindsight after observing the data. Initially studied by Hannan \cite{Hannan57}, Blackwell \cite{Blackwell56}, and Cover \cite{Cover65}, the problem of achieving low \emph{regret} for all sequences has received much attention in the last two decades, and we refer the reader to \cite{CesaBianchiLugosi06book} for a comprehensive treatment. As we show in this section, a strategy that exponentially down-weighs the decisions with large costs is a good regret-minimization strategy, and this exponential form is amenable to the sampling technique of this paper whenever the costs are convex.

More specifically, let $\K\subset\reals^d$ be a convex compact set of decisions of the learner. Let $\loss_1,\ldots,\loss_T$ be a sequence of unknown cost functions $\loss_t:\K\to\reals$. On round $t$, the learner chooses a distribution (or, a {\em mixed strategy}) $\mu_{t-1}$ supported on $\K$ and ``plays'' a decision $Y_t \sim \mu_{t-1}$.\footnote{The index $t-1$ on  $\mu_{t-1}$ reflects the fact that $Y_t$ is chosen without the knowledge of $\loss_t$.} Nature then reveals the next cost function $\loss_t$. For example, in the well-studied problem of sequential probability assignment, the Statistician predicts the probability $x_t\in [0,1]=\K$ of the next outcome $\{0,1\}$ and incurs the cost $\loss_t(x_t) = |x_t-y_t|$ with respect to the actual outcome $y_t$. A randomized strategy $Y_t$ then incurs a cost $\loss_t(Y_t)$.

The goal of the learner is to minimize \emph{expected regret} 
$$\Reg_T(U) \deq \E\left[\sum_{t=1}^T \loss_t(Y_t)-\sum_{t=1}^T \loss_t(U)\right]$$
with respect to all randomized strategies defined by $p_U \in {\mathcal P}$, for some collection of distributions ${\mathcal P}$.  %If ${\mathcal P}$ contains Dirac delta distributions $\delta_x$, the comparator term is the best fixed decision $x\in\K$ chosen in hindsight. 
A procedure that guarantees sublinear growth of regret with respect to any distribution $p_U\in {\mathcal P}$ and for any sequence of cost functions $\loss_1,\ldots,\loss_T$ will be called {\em consistent} with respect to ${\mathcal P}$. 

Define the cumulative cost functions $L_t(x) =  \sum_{s=1}^t \loss_s(x)$, and let $\eta>0$ be a parameter called \emph{the learning rate}. Fix $R(x)$ to be some convex function that defines the prior, let 
\begin{align}
	\label{eq:def_reg_s_t}
	s_t(x)=\eta L_t(x)+R(x), ~~~~~ s_0(x)=R(x)
\end{align} 
and define the probability distributions $\mu_t$ as in \eqref{eq:def_mu_t}. It turns out that this choice of $\mu_t$ is indeed a good regret-minimization strategy, as we show next. The method is similar to the Mixture Forecaster used in the prediction context \cite{Yamanishi98,Vovk01, AzouryWarmuth01, KakadeNg05}, and for a discrete set of decisions it is known as the celebrated Exponential Weights Algorithm \cite{Vovk90,LitWar94}. 

Let $D(p||q)$ stand for the Kullback-Leibler (KL) divergence between distributions $p$ and $q$. 
\begin{lemma}
	\label{lem:regret_randomized}
	For each $t\geq 1$, let $Y_t$ be a random variable with distribution $\mu_{t-1}$ as defined in \eqref{eq:def_mu_t}. The expected regret with respect to $U$ with distribution $p_U$ is
	$$\Reg_T(U)
	= \eta^{-1}\left(D(p_U||\mu_0)-D(p_U||\mu_T) \right) + \eta^{-1}\sum_{t=1}^T D(\mu_{t-1} || \mu_t).$$
	Specializing to the case $\loss_t:\K \mapsto [0,1]$ for all $t$,
	$$\Reg_T(U) \leq \eta^{-1} D(p_U||\mu_0) + T \eta/8.$$
\end{lemma}

	Before proceeding, let us make a few remarks. First, if the KL divergence between the comparator distribution $p_U$ and the prior $\mu_0$ is bounded for all $p_U\in{\mathcal P}$, the second statement of the lemma yields consistency and, even stronger, a $O(\sqrt{T})$ rate of regret growth (by choosing $\eta$ appropriately). To bound the  divergence between a continuous initial $\mu_0$ and a point distribution at some $x^*\in\K$, the analysis can be carried out in two stages: comparison to a ``small-covariance'' Gaussian centered at $x^*$, followed by an observation that the loss of the ``small-covariance'' Gaussian strategy is not very different from the loss of the deterministic strategy $x^*$. This analysis can be found in \cite[p. 326]{CesaBianchiLugosi06book} and gives a near-optimal $O(\sqrt{T\log T})$ regret bound.

	We defer the easy proof of Lemma~\ref{lem:regret_randomized} to Section~\ref{sec:proofs}. Having exhibited a good prediction strategy, a natural question is whether there exists a computationally efficient algorithm for producing a random draw from a distribution close to the desired mixed strategy $\mu_{t-1}$. To this end, we use the sampling method proposed in this paper.
	
	As a concrete example, consider linear functions $\loss_1,\ldots,\loss_T$ and let $R\equiv 0$. For simplicity assume boundedness $\loss_t:\K\mapsto[0,1]$. In this case, we may choose $\eta=\mathcal{O}(1/\sqrt{T})$. Then 
	$$ \beta_t \leq \exp\left\{2\eta \|\loss_{t}\|_\K\right\} \leq 1+C\eta$$
	for large enough $T$. Further, we set $r_t=1/d$ according to Condition~\ref{cond:lin}, and the requirement \eqref{eq:cor_relation_one_step} is seen to be satisfied for large enough $T$. With these choices of the parameters, the sequence of distributions $\mu_1,\ldots,\mu_t$ can be tracked with only one step of a random walk per iteration. The quality of this approximation is $\mathcal{O}\left(\eta d^3\nu^2\right)$ at each step.
Therefore, regret of the proposed random walk method is within $\mathcal{O}\left(T\eta d^3\nu^2\right)$ from the ideal procedure of Lemma~\ref{lem:regret_randomized}, as can be seen by writing
 	\begin{align*}
	\left|\E \loss_t(Y_t) - \E \loss_t(X_{t-1,1}) \right| \leq \int_{x\in\K} \left|\loss_t(x)|\cdot |d\sig{t-1}{1}(x) - d\mu_{t-1}(x) \right| \leq C \eta d^3 \nu^2 \ .
	\end{align*}
By choosing $\eta = \frac{1}{d^{3/2}\nu\sqrt{T}}$,
\begin{align}
	\label{eq:reg}
	\Reg_T(U) \leq Cd^{3/2}\nu D(p_U||\mu_0) \sqrt{T}.
\end{align}
A similar results holds for nonzero $R$, under the assumption that the $L_2$ distance between $d\mu_0(x)\propto \exp\{-R(x)\}dx$ and the uniform distribution on $\K$ is bounded. 

We now discuss interesting parallels between the proposed randomized method and the known deterministic optimization-based regret minimization methods. First, the statement of Lemma~\ref{lem:regret_randomized} bears striking similarity to upper bounds on regret in terms of Bregman divergences for the Follow the Regularized Leader and Mirror Descent methods \cite{lecturenotes08, BecTeb03}, \cite[Therem 11.1]{CesaBianchiLugosi06book}. Yet, the randomized method operates in the (infinite-dimensional) space of distributions while the deterministic methods work directly with the set $\K$.
Second, deterministic methods of online convex optimization face the difficulty of projections back to the set $\K$. This issue does not arise when dealing with distributions, but instead translates into the \emph{difficulty of sampling}. We find these parallels between sampling and optimization intriguing. Third, a single step of the proposed random walk requires sampling from a Gaussian distribution with covariance given by the Hessian of the self-concordant barrier. This step can be implemented efficiently whenever the Hessian can be computed. The computation time exactly matches \cite[Algorithm 2]{AbeHazRak08colt}: it is the same as time spent inverting a Hessian matrix, which is $\mathcal{O}(d^3)$ or less. Finally, as already mentioned, the idea of following a time-varying distribution is inspired by the method of following the central path in the theory of interior point methods \cite{NNbook,Nemirovski04lectures}. Similarly to the fast convergence of the chain under the lower bound on conductance, one has fast quadratic local convergence of interior point methods. One may therefore make parallels between conductance and local curvature. A further investigation of these connections is needed, especially in view of the recent developments on positive Ricci curvature of Markov chains \cite{ollivier2009ricci}.

\section{Proofs}
\label{sec:proofs}

\begin{lemma}
	\label{lem:hstep2}
	For any $t$ and $i\geq 0$, it holds that
	\begin{align*}
		\norm{\frac{d\sig{t}{i}}{d \mu_{t}} - 1 }_{t}  \leq \beta_{t}^{3/2} \norm{ \frac{d\sig{t}{i}}{d\mu_{t-1}}  - 1 }_{t-1} + \sqrt{\beta_{t}}(\beta_{t}-1)
	\end{align*}
	and, alternatively, 
	\begin{align*}
		\norm{\frac{d\sig{t}{i}}{d \mu_{t}} - 1 }_{t}  \leq \beta_{t}^{1/2} \norm{ \frac{d\sig{t}{i}}{d\mu_{t-1}}  - 1 }_{t-1} + \sqrt{\beta_{t}-1}
	\end{align*}
\end{lemma}
\begin{proof}
	Let us use the shorthand $d\sigma = d\sig{t+1}{i}$ and $\beta=\beta_{t+1}$. Using \eqref{eq:bounded_ratio}, we may write
\begin{align*}
	\norm{\frac{d\sigma}{d \mu_{t+1}} - 1 }_{t+1} &\leq \sqrt{\beta} \norm{\frac{d\sigma}{d \mu_{t+1}} - 1 }_{t} \\
	&\leq \sqrt{\beta}\left( \left\|\frac{d\sigma}{d \mu_{t+1}} - 1 \right\|_{t} - \left\|\frac{d\sigma}{d\mu_t} - 1 \right\|_{t} +\left\|\frac{d\sigma}{d \mu_{t}} - 1 \right\|_{t} \right) \ .
\end{align*}
By the triangle inequality,
\begin{align*}
	\left\|\frac{d\sigma}{d \mu_{t+1}} - 1 \right\|_{t} - \left\|\frac{d\sigma}{d\mu_t} - 1 \right\|_{t} \leq   \left\|\frac{d\sigma}{d \mu_{t+1}} - \frac{d\sigma}{d\mu_t} \right\|_{t} \ .
\end{align*}
For any function $f : \K \ra \R$, let $f^+(x) = \max(0, f(x))$ and $f^-(x) = \min(0, f(x))$. In view of \eqref{eq:bounded_ratio},
\begin{align*}
	\left\|\frac{d\sigma}{d \mu_{t+1}} - \frac{d\sigma}{d\mu_t} \right\|_{t}^2 &= \left\|\left(\frac{d\sigma}{d \mu_{t+1}} - \frac{d\sigma}{d\mu_t}\right)^+ \right\|_{t}^2
+ \left\|\left(\frac{d\sigma}{d \mu_{t+1}} - \frac{d\sigma}{d\mu_t}\right)^- \right\|_{t}^2\\
%&= \left\|\frac{d\sig{t+1}{}}{d \mu_{t}}\left(1 - \frac{d\mu_t}{d\mu_{t+1}}\right)^+ \right\|_{t+1}^2 + \left\|\frac{d\sig{t+1}{}}{d \mu_{t}}\left(1 - \frac{d\mu_{t}}{d\mu_{t+1}} \right)^- \right\|_{t+1}^2\\
&\leq \left\|\frac{d\sigma}{d \mu_{t}}(\beta-1) \mathbf{1}\left[1 < \frac{d\mu_t}{d\mu_{t+1}}\right] \right\|_{t}^2
+ \left\|\frac{d\sigma}{d \mu_{t}} \left(1-\frac{1}{\beta}\right) \mathbf{1}\left[1 \geq \frac{d\mu_t}{d\mu_{t+1}}\right]  \right\|_{t}^2\\
& \leq (\beta-1)^2  \left\|\frac{d\sigma}{d \mu_{t}} \right\|_{t}^2 \ .
\end{align*}
Therefore,
\begin{align*}
\left\|\frac{d\sigma}{d \mu_{t+1}} - 1 \right\|_{t} - \left\|\frac{d\sigma}{d\mu_t} - 1 \right\|_{t}
&\leq   (\beta-1)  \left\|\frac{d\sigma}{d \mu_{t}} \right\|_{t} \leq (\beta-1)\left(1 + \left\| \frac{d\sigma}{d\mu_t}  - 1 \right\|_{t} \right) \ .
\end{align*}
The first statement follows by rearranging the terms.

Alternatively, we can obtain an inequality that is slightly weaker for $\beta-1\approx 0$ and stronger for large $\beta$  by simply writing
\begin{align*}
	\norm{\frac{d\sigma}{d\mu_{t+1}}-1}_{t+1}^2 &= \int_{\K} \left( \frac{d\sigma}{d\mu_{t+1}} -1 \right)^2 d\mu_{t+1} = \int_{\K} \frac{d\sigma^2}{d\mu_{t+1}} -1 = \int_{\K} \frac{d\sigma^2}{d\mu_{t}^2} \frac{d\mu_{t}}{d\mu_{t+1}} d\mu_{t} -1 \ .
\end{align*}
Using $\beta$ as an upper bound on the one-sided change $\|d\mu_t/d\mu_{t+1}\|_\K$ leads to 
\begin{align*}
	&\beta\int_{\K} \frac{d\sigma^2}{d\mu_{t}^2} d\mu_{t} -1 = \beta\left\|\frac{d\sigma}{d\mu_{t}} -1\right\|^2_{t} + \beta-1
\end{align*}
and subadditivity of the square root function concludes the proof.
\end{proof}

\begin{proof}[\textbf{Proof of Theorem~\ref{thm:LVanalog}}]

	Given interior points $x,y$ in  $int(\K)$, suppose $p,q$ are the ends of the chord in $\K$ containing $x,y$ and $p,x,y,q$ lie in that order. Denote the \emph{cross ratio} by
	$$\sigma(x,y) = \frac{|x-y||p-q|}{|p-x||q-y|},$$ 
	and for two sets $S_1$ and $S_2$ let
	$$\sigma(S_1, S_2) \deq \inf_{x \in S_1, y \in S_2} \sigma(x, y).$$
	A result due to Lov\'{a}sz and Vempala \cite{LovVem07geometry} states the following. If $S_1$ and $S_2$ are measurable subsets of $\K$ and $\mu$ a probability measure supported on $\K$ that possesses a density whose logarithm is concave, then
	\begin{align*}
		\mu((\K \setminus S_1) \setminus S_2) \geq  \sigma(S_1, S_2) \mu(S_1) \mu(S_2).
	\end{align*}
	This is a non-trivial isoperimetric inequality which says that for any partition of the convex set $\K$ into $S_1,S_2$ and $S_3$, the ``volume'' of $S_3$ is large relative to that of $S_1$ and $S_2$ whenever $S_1$ and $S_2$ are separated. Given this isoperimetric result, to prove the theorem it only remains to show that the $\sigma$-distance can be lower bounded (up to a multiplicative constant) by the Riemannian metric $\rho$. The proof of this fact goes through the Hilbert (projective) metric, which is defined by 
	$$d_H(x, y) \deq \ln\left(1 + \sigma(x, y)\right).$$ 
	Further, for $x \in \K$ and a vector $v$, let 
	$$|v|_x \deq \sup\limits_{x \pm \alpha v \in \K} \alpha. $$ 
	The following two relations between the introduced notions hold. The first one (see Nesterov and Nemirovskii \cite[Theorem 2.3.2 (iii)]{NNbook}) is 
	\begin{align}
		\label{eq:mink_and_dikin_relation}
		|h|_x \leq \|h\|_x \leq 2(1 + 3 \nu)|h|_x
	\end{align}
	for all $h \in \R^d$ and $x \in int(\K)$, where $\nu$ is the self-concordance parameter of $F$. The second relation (see Nesterov and Todd \cite[Lemma 3.1]{NesterovTodd08}) states that 
		\begin{align}
			\label{1lNT}
			\|x - y\|_x - \|x - y\|_x^2 \leq \rho(x, y) \leq - \ln (1 -  \|x -y\|_x).
		\end{align}
		whenever $\|x - y\|_x < 1$.

	For any $z$ on the segment $\overline{xy}$ an easy  computation shows that $d_H(x, z) + d_H(z, y) = d_H(x, y)$. Therefore it suffices to prove the result infinitesimally. From \eqref{1lNT}, $\lim_{y \ra x}\frac{\rho(x, y)}{\|x - y\|_x} = 1,$ and a direct computation shows that
	 $$\lim_{y \ra x}\frac{d_H(x, y)}{|x - y|_x} = \lim_{y \ra x}\frac{\sigma(x, y)}{|x - y|_x} \geq 1.$$
	Hence, in view of \eqref{eq:mink_and_dikin_relation}, the Hilbert metric and the Riemannian metric satisfy
	$$\rho(x, y) \leq 2(1 + 3\nu)  d_H(x,y) .$$
	Using $\ln(1+x) \leq x$ concludes the proof.
\end{proof}

\begin{proof}[\textbf{Proof of Lemma~\ref{1lcond}}]
    The argument roughly follows the standard path, which is explained, for instance, in \cite{Vempala05survey}. Let $S_1$ be a measurable subset of $\K$ such that $\mu(S_1) \leq \frac{1}{2}$  and $S_2 = \K \setminus S_1$ be its complement.  Fix a $C>1$ and let 
	$$S_1' = S_1 \cap \{x \big| \Ps_x(S_2) \leq 1/C\} ~~\mbox{ and }~~
	S_2' = S_2 \cap \{y \big| \Ps_y(S_1) \leq 1/C\}.$$ 
	That is, points in the set $S_1'$ are unlikely to transition to the set $S_2$, and $S_2'$ is analogously unlikely to reach $S_1$ in one step. By the reversibility of the chain, which is easily checked,
 $$\int_{S_1}  \Ps_x(S_2) d\mu(x)  =  \int_{S_2}  \Ps_y(S_1) d\mu(y).$$
For any $x \in S_1'$ and $y \in S_2'$, $$d_{TV}(\Ps_x, \Ps_y) = 1 - \int_\K \min\left( \frac{d\Ps_x}{d \mu}(w), \frac{d\Ps_y}{d\mu}(w)\right)d\mu(w) \geq 1 - \frac{1}{C}.$$
That is, the transition probabilities for a pair in $S_1'$ and $S_2'$ must be dissimilar. But Lemma~\ref{1lem:gp} implies that if
$\rho(x, y) \leq \frac{ r}{C \sqrt{d}}$, then $d_{TV}(\Ps_x, \Ps_y) \leq 1 -  \frac{1}{C}$.
Therefore 
\begin{align*} 
	\rho(S'_1, S'_2) \geq \frac{r}{C \sqrt{d}}.
\end{align*}
We conclude that the sets $S'_1$ and $S'_2$ must be well-separated.
Therefore, the isoperimetric result of Theorem~\ref{thm:LVanalog} implies that 
\begin{align*}
	\mu((\K
\setminus S_1') \setminus S_2')  \geq \frac{\rho(S'_1, S'_2)}{2(1 + 3 \nu)} \min(\mu(S_1'), \mu(S_2')) \geq \frac{r}{C \nu \sqrt{d}} \min(\mu(S_1'), \mu(S_2')). \nonumber 
\end{align*} 
First suppose $ \mu(S_1') \geq (1 - \frac{1}{C})\mu(S_1)$ and $ \mu(S_2') \geq (1 - \frac{1}{C}) \mu(S_2)$. Then,
\begin{align*}
	\int_{S_1} \Ps_x(S_2) d\mu(x) &= \frac{1}{2}\int_{S_1} \Ps_x(S_2) d\mu(x) + \frac{1}{2}\int_{S_2} \Ps_x(S_1) d\mu(x) \\
	&\geq \frac{1}{2C}\mu((\K \setminus S_1') \setminus S_2') \\
	&\geq \frac{r}{2C^2\nu\sqrt{d}}\min(\mu(S_1'), \mu(S_2')) \\
	&\geq \frac{1-1/C}{2C^2} \frac{r}{\nu\sqrt{d}}\min(\mu(S_1), \mu(S_2)) ,
\end{align*}
proving the result. Otherwise, without loss of generality, suppose $\mu(S_1') \leq (1 - \frac{1}{C})\mu(S_1)$. Then 
\begin{align*}
	\int_{S_1}  \Ps_x(S_2) d\mu(x) &= \frac{1}{2}\int_{S_1} \Ps_x(S_2) d\mu(x) + \frac{1}{2}\int_{S_2} \Ps_x(S_1) d\mu(x) \\ 
	&\geq \frac{1}{2}\int_{S_1\setminus S'_1} \Ps_x(S_2) d\mu(x) \geq \frac{\mu(S_1)}{2C^2},
\end{align*}
concluding the proof.
\end{proof}

\begin{proof}[\textbf{Proof of Lemma~\ref{lem:l2change_bdd}}] The proof closely follows that in \cite{kalai2006simulated}. By definition,
	\begin{align*}
		\left\|d\mu_t/d\mu_{t+1}  \right\|_{t+1}^2 &= \int_{\K} \left(\frac{d\mu_t}{d\mu_{t+1}}\right)^2 d\mu_{t+1} = \int_{\K} \frac{d\mu_t^2}{d\mu_{t+1}} = \int_{\K} \frac{\exp\{-2s_t\} }{Z_t^2}\cdot \frac{Z_{t+1}}{\exp\{-s_{t+1}\}} \ .
	\end{align*}
	Writing out the normalization terms, 
	\begin{align*}
		\left\|d\mu_t/d\mu_{t+1}  \right\|_{t+1}^2 = \frac{\int_{\K} \exp\{-s_{t+1}\} \int_{\K} \exp\{s_{t+1}-2s_t\} }{\left(\int_{\K} \exp\{-s_{t}\}\right)^2} = \frac{Y(1)Y(-1+2(1-\delta))}{Y(1-\delta)Y(1-\delta)} 
	\end{align*}
	where $Y(a) = \int_{\K} \exp\{-a s_{t+1}\}$. As shown in \cite[Lemma 3.1]{kalai2006simulated}, the function $a^d Y(a)$ is log-concave in $a$, and thus
	$$\frac{Y(a)Y(b)}{Y\left(\frac{a+b}{2}\right)^2}\leq \left(\frac{\left(\frac{a+b}{2}\right)^2}{ab}\right)^d \ .$$
	Applying this inequality with $a=1$ and $b=-1+2(1-\delta)$,   
	\begin{align*}
		\left\|d\mu_t/d\mu_{t+1}  \right\|_{t+1}^2 \leq \left( 1+ \frac{\delta^2}{1-2\delta}\right)^d  \ .
	\end{align*}
	In particular, if $\delta\leq d^{-1/2} \leq 1/3$ (that is, $d> 8$), we obtain an upper bound of 
	$\exp\left\{\frac{d}{d-2\sqrt{d}}\right\} \leq 21$.
\end{proof}

\begin{proof}[\textbf{Proof of Lemma~\ref{lem:regret_randomized}}]
	\label{proof:regret_randomized}
	Observe that $D(\mu_{t-1} || \mu_t)$ can be written as
	\begin{align}
		\label{eq:divergence}
		\int_{\K} d\mu_{t-1} \log\frac{q_{t-1} Z_{t}}{Z_{t-1} q_{t}}
		= \log\frac{Z_{t}}{Z_{t-1}} + \int_{\K} \eta\loss_t (x) d\mu_{t-1}(x)
		= \log\frac{Z_{t}}{Z_{t-1}} + \eta\E\loss_t(Y_t).
	\end{align}
	Rearranging, canceling the telescoping terms, and using the fact that $Z_0=1$
	\begin{align*}
		\eta\E\sum_{t=1}^T \loss_t(Y_t) = \sum_{t=1}^T D(\mu_{t-1} || \mu_t) - \log{Z_{T}}.
	\end{align*}
	Let $U$ be a random variable with a probability distribution $p_U$. Then
	$$-\sum_{t=1}^T \E \loss_t(U) = \eta^{-1}\int_{\K} -\eta L_T(u) dp_U(u) = \eta^{-1}\int_{\K} dp_U(u)\log\frac{q_T(u)}{q_0(u)}  $$	
	Combining,
	\begin{align*}
		\E\left[\sum_{t=1}^T \loss_t(Y_t)-\sum_{t=1}^T \loss_t(U)\right] &= \eta^{-1}\int_{\K} dp_U(u)\log\frac{q_T(u)/Z_T}{q_0(u)}  + \eta^{-1}\sum_{t=1}^T D(\mu_{t-1} || \mu_t)\\
		&= \eta^{-1}\left(D(p_U||\mu_0)-D(p_U||\mu_T) \right) + \eta^{-1}\sum_{t=1}^T D(\mu_{t-1} || \mu_t).
	\end{align*}
	Now, from Eq.~\eqref{eq:divergence}, the KL divergence can be also written as
	\begin{align*}
		D(\mu_{t-1} || \mu_t) &= \log\frac{\int_\K e^{-\eta\loss_t(x)} q_{t-1} (x)dx}{\int_\K q_{t-1}(x)dx} + \eta\E\loss_t(Y_t)
%		= \log \E e^{-\eta\loss_t(X_t)} + \log e^{\E\eta\loss_t(X_t)}
		= \log \E e^{-\eta(\loss_t(Y_t)-\E\loss_t(Y_t))}
	\end{align*}
	By representing the divergence in this form, one can obtain upper bounds via known methods, such as \emph{log-Sobolev inequalities} (e.g. \cite{BouLugMas03}). In the simplest case of bounded loss, it is easy to show that $D(\mu_{t-1} || \mu_t) \leq O(\eta^2)$, and the particular constant $1/8$ can be obtained by, for instance, applying Lemma A.1 in \cite{CesaBianchiLugosi06book}. This proves the second part of the lemma.
\end{proof}

%\appendix

\section{Smooth Variation of the Transition Kernel}
\label{sec:variation}

In this section, we study the transition $x\to y$. For this purpose, it is enough to assume that $x$ is the origin and that the Dikin ellipsoid at $x$ is a unit Euclidean ball. This can be achieved by an affine transformation, leading to no loss of generality since the resulting statement about measures on $\K$ is invariant with respect to affine transformations. Hence, in what follows, for the particular $x$ we have $ <\cdot, \cdot>_x=<\cdot, \cdot>$ and $\|\cdot\|_x=\|\cdot\|$. Since $x$ is the origin, we have $\E\|z\|_x^2 = r^2$ for $z$ sampled from $\g^r_x$. Further, without loss of generality, we may also assume $s(x) = 0$. 

\begin{proof}[\textbf{Proof of Lemma~\ref{1lem:gp}}] %Without loss of generality, we  assume $C > 1$ and $r < 1$.

In view of the first inequality in Eq.~\eqref{1lNT}, 
$$\|x - y\|_x - \|x - y\|_x^2 \leq \rho(x, y) \leq \frac{ r}{C \sqrt{d}}. $$
Without loss of generality, assume $\frac{r}{C\sqrt{d}} \leq \frac{1}{8}$. First, we claim that $\|x - y\|_x$ must be small. For the sake of contradiction, suppose $\|x - y\|_x > 1/2$ and consider a point $y'$ with $\|x-y'\|_x = 1/2$ and lying on the geodesic path between $x$ and $y$ with respect to the Riemannian metric. Clearly, $\rho(x,y')\leq \frac{r}{C \sqrt{d}} \leq \frac{1}{8}$, yet by Eq.~\eqref{1lNT} we have $\frac{1}{4} \leq \rho(x,y')$, contradicting our assumption. Hence, $\|x - y\|_x \leq 1/2$, and, therefore, $\|x - y\|_x \leq \frac{2r}{C \sqrt{d}}$.

It remains to show that if $x, y \in \K $ and 
$$\|x -y\|_x \leq \frac{2r}{C \sqrt{d}},$$ then $$d_{TV}(\Ps_x, \Ps_y) = 1 -  \frac{1}{C}.$$
By definition, we have that 
\begin{align*} 
	1 - d_{TV}(\Ps_x, \Ps_y) = \E_z\left[ \min\left\{ 1, \frac{\g^r_y(z)}{\g^r_x(z)},
\frac{\g^r_z(x)\exp(s(x))}{\g^r_x(z)\exp(s(z))}, \frac{\g^r_z(y)\exp(s(y))}{\g^r_x(z)\exp(s(z))}\right\}\right],
\end{align*}
    where the expectation is taken over a random point $z$ having density $\g^r_x$. 
	Thus, it suffices to prove that for some $C>1$
    $$\p\left[\min\left\{ \frac{\g^r_y(z)}{\g^r_x(z)},  \frac{\g^r_z(x)\exp(s(x))}{\g^r_x(z)\exp(s(z))}, \frac{\g^r_z(y)\exp(s(y))}{\g^r_x(z)\exp(s(z))} \right\} > \frac{1}{C}\right] \geq  \frac{1}{C}.$$
By our assumption, $x$ is the origin and $D^2F(x) = I$, the latter implying that $V(x)=0$. Thus,
$$ \frac{\g^r_y(z)}{\g^r_x(z)} = \exp\left\{ -\frac{d\|y-z\|^2_y}{r^2} + V(y) + \frac{d\|z\|^2}{r^2} \right\},$$
$$ \frac{\g^r_z(x)\exp(s(x))}{\g^r_x(z)\exp(s(z))} = \exp\left\{ -\frac{d\|z\|^2_z}{r^2} + V(z) + \frac{d\|z\|^2}{r^2} + (s(x)-s(z)) \right\},$$
and
$$ \frac{\g^r_z(y)\exp(s(y))}{\g^r_x(z)\exp(s(z))} = \exp\left\{ -\frac{d\|y-z\|_z^2}{r^2}+V(z) + \frac{d\|z\|^2}{r^2} + (s(y)-s(z))\right\} \ .$$
Thus, it remains to prove that there exists a constant $C$ such that
\begin{align*}
	&\p\Big[
	\max\Big\{
		 d\|y - z\|_y^2 - r^2V(y), ~~~~
		  d\|z\|_z^2 + r^2(s(z)-s(x)) - r^2V(z), \\
		&\hspace{2cm} d\|z - y\|_z^2 + r^2(s(z) - s(y)) - r^2V(z)
		\Big\} <
		d\|z\|^2 + r^2 C
	\Big]  \geq \frac{1}{C}.
\end{align*}
This fact is shown in technical Lemmas~\ref{lem:term1} and \ref{lem:term2} below.
\end{proof}

In proving the technical lemmas, we will use the fact that $\|x - y\|_x \leq \frac{2r}{C \sqrt{d}}$ as shown above, and that $\|x-z\|_x$ (for $z$ sampled from $\g^r_x$) is likely to be bounded above by a multiple of $r$ by straightforward concentration arguments.

\begin{lemma}
	There exists a constant $C>0$ such that
\begin{align*} 
	\p\left[\max\left(- V(y),  - V(z)\right) < C \right] >  0.9
\end{align*}
\end{lemma}
\begin{proof}
	Fix a constant $c$. First, notice that over a Euclidean ball of radius $c/d$ around the origin, the Hessians $D^2 F(u)$ are lower-bounded by a factor  of $(1-c/d)^2$ from the Hessian at the origin (the identity) by \eqref{eq:hessian_similar}. Hence, the determinant function can decrease from 1 by at most a constant factor. Thus $-V(u) < C'$ for some constant $C'$ for any $u$ with $\|x-u\|_x \leq c/d$. Now recall that $y$ is deterministically within the $1/d$ ball, while $z$ is in the ball of radius $c/d$ with high probability.
\end{proof}

\begin{lemma}
	Under step size Condition~\ref{cond}, for any 
	\label{lem:term1}
$$\p\left[\max\Big\{ s(z)-s(x), s(z) - s(y) \Big\} <  C\right] %\geq \mathrm{erfc}\left(\frac{1}{\sqrt{2}}\right)
 > 0.32.$$
\end{lemma}
\begin{proof}
Since with large enough probability $\|x-y\|_x< C'r$ and $\|x-z\|_x< C'r$, we also have $\|z-y\|_x < 2C'r$. Then, by \eqref{eq:hessian_similar}, the norms at $z$ and $x$ are within a multiplicative constant, and thus the pairs $(z,x)$ and $(z,y)$ are subject to the step size choice specified in the condition. That is, there exists a $g$ such that
$$s(z)-s(x) = s(z)- s(x)-\inner{g,z-x} + \inner{g, z-x} \leq C + \inner{g, z-x}$$
and similarly
\begin{align*}
	s(z)-s(y) = s(z)- s(y)-\inner{g,z-y} + \inner{g, z-y} \leq C + \inner{g, z-y}
\end{align*}
Then, assuming (without loss of generality) $x=0$,
$$\Prob{\max\left\{\inner{g,z-x}, \inner{g,z-y}\right\} <  0}  = \Prob{\inner{ g, z} \leq \min\left\{ 0,  \inner{g, y}\right\} }.$$
Observe that $\inner{g, z} $ is a Gaussian random variable whose standard deviation is larger than
$\left\|g \right\|\| y\|.$
Therefore,
$$\p\left[\inner{ g, z} \leq \min\left\{ 0,  \inner{g, y} \right\}\right] \geq \mathrm{erfc}\left(1/\sqrt{2}\right) > 0.32,$$
where $\mathrm{erfc}(x) \deq \frac{2}{\sqrt{\pi}}\int_x^\infty e^{-t^2}dt$ is the usual complementary error function.

\end{proof}
The following probabilistic upper bound completes the proof.

\begin{lemma}
	\label{lem:term2} There exists a constant $C>0$ such that
	\begin{align*} 
		\p\left[\max\Big\{ \|y - z\|_y^2,  \|z
\|_z^2 ,  \|z - y\|_z^2 \Big\} - \|z\|^2 < \frac{Cr^2}{d}\right] > 0.9
	\end{align*}
\end{lemma}

\begin{proof}[\textbf{Proof of Lemma~\ref{lem:term2}}]
Since $\|y\| < \frac{Cr}{\sqrt{d}}$, $\|y\|_y$ and $\|y\|_z$ are less than $\frac{Cr}{\sqrt{d} }$. So it suffices to show that
\begin{align*}
	\p\left[\max\Big\{ \|z\|_y^2 - \|z\|^2, \|z \|_z^2 - \|z\|^2,  \inner{y, z}_y, \inner{y, z}_z \Big\}  < \frac{C r^2}{d} \right] & > 0.9
\end{align*} 
We proceed to do so by proving probabilistic upper bounds on each of the terms
$$
\mbox{(a) } \|z\|_y^2 - \|z\|^2 \ , ~~~~\mbox{ (b) } \|z \|_z^2 - \|z\|^2 \ , ~~~~~\mbox{ (c) } \inner{y, z}_y \ ,  ~~~ \mbox{ and ~~~~ (d) } \inner{y, z}_z
$$
separately, and finally applying the union
bound. We first prove an upper bound on (a) and (b). Note that $r \leq \frac{1}{d}$ and thus $r^3 \leq \frac{r^2}{d}$. It suffices to observe that by \eqref{eq:hessian_similar} 
	 $$\|z \|_z^2 - \|z\|^2 \leq \left(\left(\frac{1}{1-\|z\|}\right)^2 - 1\right)\|z\|^2 \leq 8\|z\|^3,$$ whenever $\|z\| < 1/2$. 
	Similarly, for $\|y\|< 1/2$, 
	$$\|z \|_y^2 - \|z\|^2 \leq \left(\left(\frac{1}{1-\|y\|}\right)^2 - 1\right)\|z\|^2 \leq 8\|z\|^3.$$
	There exists a constant $C$ such that the quantity $\|z\|^3$ is bounded by $Cr^3$ with probability at least $0.99$.
	
	We now turn to bounding (c) and (d). Let $[0, u]$ denote the line segment between the origin and $u$. By the mean-value theorem,
	\begin{align*} 
		\inner{y, z}_y &= \inner{y, z} + (\inner{y,z}_y - \inner{y, z}) \leq \inner{y, z} + \sup\limits_{y' \in [0, y]} D^3F(y')[y, y, z]
	\end{align*}
	\begin{align*} 
		\inner{y, z}_z & = \inner{y, z} + (\inner{y,z}_z - \inner{y, z}) \leq  \inner{y, z} + \sup\limits_{z' \in [0, z]} D^3F(z')[y, z, z] 
	\end{align*}
	Observe that 
	$$\inner{y,z}\leq \frac{C\|y\| \|z\|}{\sqrt{d}}$$
	with probability at least $0.99$ by a measure-concentration argument. Indeed, most of the vectors $z$ are almost perpendicular to the given vector $y$. 
	Now, using \eqref{eq:trice_upper},
	$$\sup\limits_{y' \in [0, y]} D^3F(y')[y, y, z]  \leq \sup\limits_{y' \in [0, y]} 2\|y\|^2_{y'} \|z\|_{y'} \leq \frac{Cr^2}{d}$$
	and
	$$\sup\limits_{z' \in [0, z]} D^3F(z')[y, z, z]  \leq \sup\limits_{z' \in [0, z]} 2\|y\|_{z'}\|z\|^2_{z'} \leq \frac{Cr^3}{\sqrt{d}}\leq \frac{Cr^2}{d}$$ 
	with probability at least $0.99$. Therefore, there exists a constant $C>0$ such that
	\begin{align*}
		\p\left[  \inner{y,z}_y < \frac{Cr^2}{d} \right]  >  0.98
	\end{align*}
	and the same statement holds for $\inner{y,z}_z$. We also have that
		$$\p\left[\frac{\|y\| \|z\|}{\sqrt{d}} + \sup_{z' \in [0, z]} 2\|y\|_{z'}\|z\|_{z'}^2  \leq \frac{C r^2}{d}\right] > 0.99 $$ Therefore, 
		\begin{align*}
			\p\left[\inner{y,z}_z  < \frac{C r^2}{d} \right] >  0.98.
		\end{align*}
\end{proof}

\section{Self-concordant barriers}
\label{sec:self_conc_def}

Let $\K$ be a convex subset of $\R^d$ that is not contained in any $(d-1)$-dimensional affine subspace and $int(\K)$ denote its interior.
 Following Nesterov and Nemirovskii, we call a real-valued function $F:int(\K) \ra \R$,  a
regular self-concordant barrier if it satisfies the conditions stated below. For convenience, if $x \not \in int(\K)$, we define $F(x) = \infty$.
\begin{enumerate} \item (Convex, Smooth) $F$ is a convex thrice continuously differentiable function on $int(\K)$. \item (Barrier) For every
sequence of points $\{x_i\} \in int(\K)$ converging to a point $x \not \in int(\K)$,  $\lim_{i \ra \infty} f(x_i) = \infty$. \item (Differential
Inequalities)    For all $h \in \R^d$ and all $x \in int(\K)$, the following inequalities hold.

\begin{enumerate} \item $D^2 F(x)[h, h]$ is $2$-Lipschitz continuous with respect to the local norm, which is equivalent to
$$D^3 F(x)[h, h, h] \leq 2 (D^2 F(x)[h, h])^{\frac{3}{2}}.$$ \item $F(x)$ is $\nu$-Lipschitz continuous with respect to the local norm defined by $F$,
$$|D F(x)[h]|^2 \leq \nu D^2 F(x)[h, h].$$ We call the smallest positive integer $\nu$ for which this holds, \emph{the self-concordance  parameter} of the barrier. 
\end{enumerate} 
\end{enumerate}
The following results can be found, for instance, in \cite{NNbook,Nemirovski04lectures,NemTod08}. First, 
\begin{align}
	\label{eq:trice_upper}
	|D^3 F(x)[h_1, \dots, h_k]|  \leq  2\|h_1\|_x \|h_2\|_x \|h_3\|_x \ .
\end{align}
Second, if $\delta = \|h\|_x <1$, then
\begin{align}
	\label{eq:hessian_similar}
	(1-\delta)^2 D^2 F(x) \preceq D^2 F(x+h) \preceq (1-\delta)^{-2} D^2 F(x) \ .
\end{align}

\bibliographystyle{plain} 
\bibliography{regret_by_sampling}

\end{document}